\documentclass{article}
\usepackage{libertine}
\usepackage{graphicx} % more modern
\usepackage{subfigure} 
\usepackage{natbib}
\usepackage{algorithm}
\usepackage{algorithmic}

\usepackage{bbm,amsthm,amssymb,amsmath,amsfonts,mathtools,comment,enumerate}

\mathtoolsset{
showonlyrefs=true%true % or false in draft mode
}

\DeclareMathOperator*{\argmax}{argmax}

\DeclareMathOperator*{\expec}{\mathbb E}

\newcommand{\indic}{{\mathbbm 1}}

\newcommand{\chosen}{{\texttt{Chosen}}}
\newcommand{\actual}{{\texttt{Actual}}}
\newcommand{\comply}{{\texttt{Comply}}}

\newcommand{\x}{{\mathbf x}}
\newcommand{\y}{{\mathbf y}}
\newcommand{\wt}{{\mathbf w}}
\newcommand{\vt}{{\mathbf v}}

\newcommand{\bP}{{\mathbf P}}

\newcommand{\cA}{{\mathcal A}}
\newcommand{\cB}{{\mathcal B}}
\newcommand{\cC}{{\mathcal C}}
\newcommand{\cH}{{\mathcal H}}

\newcommand{\fN}{{\mathfrak N}}
\newcommand{\fC}{{\mathfrak C}}
\newcommand{\fA}{{\mathfrak A}}
\newcommand{\fD}{{\mathfrak D}}

\newcommand{\eod}{{${}$\\}}

\newcommand{\regret}{{\mathtt{Regret}}}

\newcommand*\loss{\ensuremath{\boldsymbol\ell}}

\newtheorem{thm}{Theorem}

\newtheorem{prop}{Proposition}
\newtheorem{lem}{Lemma}
\newtheorem{eg}{Example}
\newtheorem{defn}{Definition}

\theoremstyle{remark}

\title{ Compliance-Aware Bandits}
\author{
  Della Penna, Nicol\'as\\
  Australian National University\\
  \texttt{n@nikete.com}
  \and
  Reid, Mark D.\\
  Australian National University\\
  \texttt{mark.reid@anu.edu.au}
   \and
  Balduzzi, David\\
  Victoria University Wellington\\
  \texttt{david.balduzzi@vuw.ac.nz}
}

\begin{document} 
\maketitle
\begin{abstract}

Motivated by clinical trials, we study bandits with observable non-compliance. 
At each step, the learner chooses an arm, after, instead of observing only the reward, it also observes the action that took place.
We show that such noncompliance can be helpful or hurtful to the learner in general.
Unfortunately, naively incorporating compliance information into bandit algorithms loses guarantees on sublinear regret.
We present hybrid algorithms that maintain regret bounds up to a multiplicative factor and can incorporate compliance information.
Simulations based on real data from the International Stoke Trial show the practical potential of these algorithms.

\end{abstract} 

%!TEX root = main.tex
%%%%%%%%%%%%%%%%%%%%%%%%%%%%%%%%%%%%%%%%%%%%%%%%%%%%%%%%%%%%%%%%%%%%%%%%%%%%%%%%%%%%%%%%
\section{Introduction}
\label{intro}

People often don't do as they are told. Approximately 50\% of patients suffering from chronic illness do not take prescribed medications \cite{sabate:03}. It is safe to assume that the rate at which patients or doctors will follow the recommendations provided by an algorithm will fall well short of 100\%. 
Unfortunately, despite its importance in medical applications \cite{vrijens:12,hugtenburg:13}, compliance has not been analyzed in the bandit literature. 

In this paper, we introduce compliance awareness into the \emph{bandit setting}. Bandit problems are concerned with optimal repeated decision-making in the presence of uncertainty \cite{robbins:52, lai:85, bubeck:12}. The main challenge is to trade-off exploration and exploitation, so as to collect enough samples to estimate the rewards from different strategies whilst also strongly biasing samples towards those actions most likely to yield high rewards.  

Our running example is an algorithm that recommends treatments to patients. For concreteness, consider a mobile app that encourages patients who have recently suffered a stroke to carry out various low intensity interventions that may be beneficial in preventing future strokes.
These could be as simple as meditating, going for a walk or taking an aspirin.
The effects of the interventions on the probability of a future stroke may be small. The social benefits of collectively choosing the most effective interventions, however, may however be large.
 
However, there are other settings in which compliance information is potentially available. For example, an algorithm could recommend treatments to \emph{doctors}. Whether or not the doctor then prescribes the recommended treatment to the patient is then extremely informative, since the doctor may make observations and have access to background knowledge that is not available to the algorithm.
A quite different setting is online advertising, where bandit algorithms are extensively applied to recommend which ad to display \cite{graepel:10,mcmahan:13}. In practice, the recommendations provided by the bandit may not be followed. For example, sales teams often have have hand-written rules that override the bandit in certain situations. %Alternatively, the algorithm may assign a user to a treatment on their laptop, and when the user is not logged in, expose him to a different treatment on their mobile.
Clearly, the bandit algorithm should be able to learn more efficiently if it is provided  with information about which ads were actually shown.

In the classic multi-armed bandit setting, the player chooses one of $k$-arms on each round and receives a reward \cite{auer:02b,auer:02}. The player is not told what the reward would have been had it chosen a different arm. The goal is to minimize the cumulative regret over a series of $T$ rounds. In the more general compliance setting, the action chosen by the algorithm is not necessarily the action that is finally carried out, see section~\ref{sec:formal}. Instead, a compliance process mediates between the algorithm's recommendation and the action that is actually taken. Importantly, the compliance process may depend on latent characteristics of the subject of the decision. We focus on the case where the outcome of the compliance process is observable.

Unfortunately, compliance information is a two-edged sword. There are settings where it is useful; but  it can also lead to linear regret. We develop bounded regret algorithms that incorporate compliance information.

\paragraph{Outline.}
Section~\ref{sec:noncompliance} introduces the formal compliance setting and introduces three protocols for incorporating compliance information into bandit algorithms. It turns out that each protocol has strengths and weaknesses. The simplest protocol ignores compliance information -- which yields the classical setting where standard regret bounds hold. If, instead of attending to its recommendations, the bandit attends to what whether the patient actually takes the treatment, then it is possible, in some scenarios, to learn faster than without compliance information. On the other hand, there are no guarantees on convergence when an algorithm attends purely to the compliance of patients and ignores its own prior recommendations -- examples of linear regret are provided in section~\ref{sec:protocols}. 

A natural goal is thus to simultaneously incorporate compliance information whilst preserving the no-regret guarantees of the classical setting. Section~\ref{sec:algol} presents two hybrid algorithms that do both. The first, \texttt{HierarchicalBandit} is in a two-level bandit algorithm. The bottom-level learns three experts that specialize on difference kinds of compliance information. The top-level is another bandit that learns which expert performs optimally. The algorithm thus has no-regret against both the treatments and two natural reward protocols that incorporate compliance information. The second algorithm, \texttt{ThompsonBounded}, rapidly converges to Thompson sampling with standard guarantees. However, when Thompson sampling is unsure about which arm to pull, the algorithm takes advantage of the uncertainty to introduce arm-pulls sampled from \texttt{HierchicalBandit}.

Empirically, \texttt{ThompsonBounded} achieves a surplus of 8.9 extra survivals (that is, human lives) relative to the randomized baseline.
The \texttt{HierarchicalBandit} algorithm with \texttt{Epsilon Greedy} as the base algorithm achieves a surplus of 9.2.
In contrast, the best performing strategy that is not compliance aware is Thompson sampling, which yields 7.9 extra survivals.

\paragraph{Comparison with other bandit settings.}
It is useful to compare noncompliance with other bandit settings. Partial monitoring and its generalizations, such as feedback graphs, are concerned with situations where the player only partial observes its loss \cite{alon:15}. Our setting is an extension of the bandit setting, where additional compliance-information is provided. Whether or not a patient complies is a form of side-information. However, in contrast to the side-information available to contextual bandits, it is only available \emph{after} an arm is pulled. An interesting question, left for future work, is how contextual and compliance information can both be incorporated into bandit algorithms.

Hybrid algorithms were previously proposed in the best-of-both-worlds scenario \cite{bubeck:12a,seldin:14}, where the goal is to construct a bandit that plays optimally in both stochastic and adversarial environments. Vapnik introduced a related notion of side-information into the supervised setting with his learning under privileged information framework \cite{vapnik:09}. Perhaps the closest setting to ours are the confounded bandits in \cite{bareinboim:15}, see section~\ref{sec:confounded}.

%!TEX root = main.tex
%%%%%%%%%%%%%%%%%%%%%%%%%%%%%%%%%%%%%%%%%%%%%%%%%%%%%%%%%%%%%%%%%%%%%%%%%%%%%%%%%%%%%%%%
\section{Models of Noncompliance}
\label{sec:noncompliance}

This section introduces a formal setting for bandits with noncompliance and introduces protocols that prescribe how to make use of compliance information. Before diving into the formalism let us discuss, informally, how compliance information can be useful. 

First, suppose that the patient population is homogeneous in their response to the treatment, and that patients take the treatment with probability $p$ if prescribed and probability $1-p$ otherwise where $p<0.5$. In this setting, it is clear that a bandit algorithm will learn faster by rewarding arms according to whether the treatment was \emph{taken} by the patient, rather than whether it was \emph{recommended} to the patient. 

As a second example, consider \emph{corrective compliance} where patients who benefit from a treatment are more likely to take it, since they have access to information that the bandit does not. The bandit clearly benefits by learning from the information expressed in the behavior of the patients. Learning from the treatment actually taken is therefore more efficient than learning from the bandit's recommendations. Further examples are provided in section~\ref{sec:formal}.

%%%%%%%%%%%%%%%%%%%%%%%%%%%%%%%%%%%%%%%%%%%%%%%%%%%%%%%%%%%%%%%%%%%%%%%%%%%%%%%%%%%%%%%%
\subsection{Unobserved confounders.}
\label{sec:confounded}

%\cite{bareinboim:15}: ``The study of unobserved confounders is one of the central themes in the modern literature of causal inference. To appreciate the challenges posed by these confounders, consider the comparison between a randomized clinical trial conducted by the Food and Drug Administration (FDA) versus physicians prescribing drugs in their offices. A key tenet in any FDA trial is the use of randomization for the treatment assignment, which precisely protects against biases that might be introduced by physicians. Specifically, physicians may prescribe Drug A for their wealthier patients who have better nutrition than their less wealthy ones, when unknown to the doctors, the wealthy patients would recover without treatment. On the other hand, physicians may avoid prescribing the expensive Drug A to their less privileged patients, who (again unknown to the doctors) tend to suffer less stable immune systems causing negative reactions to the drug. If a naive estimate of the drug’s causal effect is computed based on physicians' data (obtained through random sampling, but not random assignment), the drug would appear more effective than it is in practice -- a bias that would otherwise be avoided by random assignment. Confounding biases (of variant magnitude) appear in almost any application in which the goal is to learn policies (instead of statistical associations), and the use of randomization of the treatment assignment is one established tool to combat them.''

An important point of comparison is the bandits with unobserved confounders model introduced in \cite{bareinboim:15}. That paper was motivated using an extended example involving two subpopulations (drunk and sober) gambling in a casino. Since we are primarily interested in clinical applications, we map their example onto two subpopulations of patients, rich and poor. Suppose that rich patients always take the treatment (since they can afford it) and that they are also healthier in general. Poor patients only take the treatment when prescribed by a doctor.

Barenboim et al observe that the question ``what is the patient's expected reward when taking the treatment (formally: $\expec[R|{A=1}]$)?'' is confounded by the latent variable \texttt{wealth}. Estimating the effect of the treatment -- which may differ between poor and rich patients -- requires more refined questions. In our notation: 
``what is the patient's expected reward when taking the treatment, given she is wealthy (formally: $\expec[R|{A=1}, \text{always-taker}]$)?'' and  ``what is the patient's expected reward when taking the treatment, given she is poor (formally: $\expec[R|{A=1}, \text{complier}]$)'', see example~\ref{eg:rich}.

The solution proposed in \cite{bareinboim:15} is based on the regret decision criterion (RDC), which estimates the optimal action according to $\argmax_{a}\expec[R|A=a,\text{patient's inclination}]$, where the action chosen, $A=a$, may \emph{differ} from the patient's latent inclination. Essentially, computing the RDC requires imposing interventions via the $do(\cdot)$ operator. However, overruling a patient or doctor's decision is often impossible and/or unethical in clinical settings. The counterfactual information required to compute the RDC may therefore not be available in practice.

Compliance information does not act as a direct substitute for the $do(\cdot)$ operator. However, compliance information is often readily available and, as we show below, can be used to ameliorate the effect of confounders by giving a partial view into the latent structure of the population that the bandit is interacting with.

\subsection{Formal setting}
\label{sec:formal}
  
%Given this model, we define the ``noncompliance level'' $n(c, u)$ for a specific choice $C=c$ and latent variables $U=u$ to mean the probability that $A \neq c$ given those values, that is, $n(c, u) = 1 - P(A=c|C=c, U=u)$.

More formally, we consider a sequential decision making problem where a process mediates between the actions chosen by the algorithm and the action carried out in the world. The general game is as follows:

\begin{defn}[bandit with compliance information]\label{def:compliance_bandit}\eod
	At each time-step $t$, the player selects an action $c_t\in \cA=[k]=\{1,\ldots,k\}$ (the chosen action). The environment responds by carrying out an action $a_t\in\cA=[k]$ (the actual action) and providing reward $r_t\in[0,1]$, or loss $\ell_t$.

	The standard bandit setting is the special case where $a_t$ is either unobserved or $c_t = a_t$ for all $t\in[T]$.
\end{defn}

Compliance and outcomes are often confounded. For example, healthy patients may be less inclined to take a treatment than unhealthy patients. 
The set of compliance-behaviors is the set of functions $\cC=\{\nu:\cA\rightarrow\cA\}$ from advice to treatment-taken \cite{koller:09}. 

\begin{defn}[model assumptions]\label{def:assumptions}\eod
	We make the following assumptions:
	\begin{enumerate}
		\item Compliance-behavior $\nu(u)\in\cC$ depends on a latent variable sampled i.i.d. from unknown  $\bP(u)$.
		\item Outcomes $r(\nu(u), a,u)$ depend on compliance-behavior, treatment-taken and the latent $u$. That is, outcomes are a fixed function $r:\cC\times \cA\times U\rightarrow[0,1]$.		
	\end{enumerate}
\end{defn}
%The confounding of compliance with outcomes is modeled with a latent variable $u$ such that compliance $\nu(u)$ depends on $u$; and outcomes $r(c,\nu(u),u)$ is a function $r:\cA\times (\cA\rightarrow \cA)\times U\rightarrow [0,1]$ that depends on treatment-advice, compliance and the latent variable $u$.

When $|\cA|=k=2$ (corresponding to control and treatment), we can list the compliance-behaviors explicitly.
\begin{defn}[compliance behaviors]\label{def:compliance_model}\eod
	For $k=2$, the following four subpopulations capture all deterministic compliance-behaviors:
	\begin{align}
		\text{never-takers ($\fN)$:}     &\quad c_0\mapsto a_0\qquad c_1\mapsto a_0\\
		\text{always-takers ($\fA)$:}    &\quad c_0\mapsto a_1\qquad c_1\mapsto a_1\\
		\text{compliers ($\fC)$:} &\quad c_0\mapsto a_0\qquad c_1\mapsto a_1\\ 
		\text{defiers ($\fD)$:}   &\quad c_0\mapsto a_1\qquad c_1\mapsto a_0
	\end{align}
	Let $p_s:= \expec_{u\sim \bP(u)}[\indic_{[\nu(u)=s]}]$ denote the probability of sampling from subpopulation $s\in\{\fN,\fA,\fC, \fD\}$.
\end{defn}
Unfortunately, the subpopulations cannot be distinguished from observations. For example, a patient that takes a prescribed treatment may be a complier or an always-taker. Nevertheless, observing compliance-behavior provides potentially useful side-information. The setting differs from contextual bandits because the side-information is only available \emph{after} the bandit chooses an arm.

\begin{defn}[stochastic reward model]\label{def:reward_model}\eod
	The expected reward given subpopulation $s$ and the actual treatment $j\in\cA$ is
	\begin{equation}
		r_{s,j} 
		:= \expec_{u\sim \bP(u)}\big[r(\nu(u),a_j,u)\,\big|\, \nu(u)=s\big]
		\label{eq:exp_rew}		
	\end{equation}	
	for $s\in \{\fN, \fA,\fC,\fD\}$.
\end{defn}
The goal of the player is to maximize the cumulative reward received,
i.e. choose a sequence of actions $(c_t)_{t\in T}$ that maximizes $\expec_{u\sim \bP(u)}\left[\sum_t r(\nu(u),\nu(u)(c_t), u)\right]$. We quantify the performance of algorithms in terms of regret, which compares the cumulative reward against that of the best action in hindsight.

%%%%%%%%%%%%%%%%%%%%%%%%%%%%%%%%%%%%%%%%%%%%%%%%%%%%%%%%%%%%%%%%%%%%%%%%%%%%%%%%%%%%%%%%
\subsection{Reward protocols}
\label{sec:protocols}

Since compliance-information is only available after-pulling an arm, it cannot be used directly when selecting arms. However, how compliance-information can be used to modify the updates performed by the algorithm. For example, if the bandit recommends taking a treatment, and the patient does not do so, we have a choice about whether to update the arm that the bandit recommended (treatment) or the arm that the patient pulled (control).

 %Each protocol can be combined with any multi-armed bandit algorithm. 

\begin{defn}[reward protocols]\label{def:protocols}\eod
	We consider three protocols for assigning rewards to arms:
	\begin{enumerate}[P1.]
		\item \textbf{\chosen: chosen-treatment updates.}\\
		Assign reward $r_t$ to arm $j$ if $c_t=j$.
		\item \textbf{\actual: actual-treatment updates.}\\
		Assign reward $r_t$ to arm $j$ if $a_t=j$.
		\item \textbf{\comply: compliance-based updates.}\\
		Assign reward $r_t$ to arm $j$ if $c_t=j$ and $a_t=j$.
	\end{enumerate}
\end{defn}
Each protocol has strengths and weaknesses.

\paragraph{Protocol \#1: \chosen.}
Under \chosen, the bandit advises the patient on which treatment to take, and ignores whether or not the patient complies. 

\begin{prop}\label{prop:chosen}
	Standard regret bounds hold for any algorithm under \chosen.
\end{prop}

\begin{proof}
	The regret bound for any bandit algorithm holds since the setting is the standard bandit setting.
\end{proof}

\paragraph{Protocol \#2: \actual.} 
Expected rewards depend on the treatment Eq.~\eqref{eq:exp_rew} chosen by the patient, and not directly on the arm pulled by the bandit. Thus, a natural alternative to \chosen\, is \actual, where the bandit assigns rewards to the treatment that the patient actually used -- which may not in general coincide with the arm that the bandit pulled.

\begin{prop}\label{prop:actual}
	There are settings where \actual\, outperforms \chosen\, and \comply.
\end{prop}
\begin{proof}
	Suppose that $r_{s,j}=r_j$ depends on the treatment but not the subpopulation. Further suppose the population is a mix of always-takers, never-takers, and compliers -- but no defiers. 
	Always-takers and never-takers ignore the bandit, which therefore only interacts with the compliers. 

	The rewards used to update \chosen\, are, in expectation 
	\begin{align}
		\expec[R\,|\,c_0] 
		& = (1-p_\fA)\cdot r_0 + p_\fA\cdot r_1\\
		\expec[R\,|\,c_1]
		& = p_\fN\cdot r_0 + (1-p_\fN)\cdot r_1
	\end{align}
	whereas the rewards used to update \actual\, are
	\begin{equation}
		\expec[R\,|\,a_0] = r_0
		\text{ and }
		\expec[R\,|\,a_1] = r_1.
	\end{equation}
	It follows that
	\begin{align}
		r_{\fC,0} &  = \expec[R\,|\,a_0] \neq \expec[R\,|\,c_0] \text{ and}\\
		r_{\fC,1} &  = \expec[R\,|\,a_1] \neq \expec[R\,|\,c_1].
	\end{align}
	Thus, \actual\, assigns rewards to arms based on their effect on compliers (which are the only subpopulation interacting with the bandit), whereas the rewards assigned to arms by \chosen\, are diluted by patients who do not take the treatment. Finally, \actual\, outperforms \comply\, because it updates more frequently.
\end{proof}

However, \actual\, can fail completely.
\begin{eg}[\actual\, has linear regret; defiers]\label{eg:defiers}\eod
	Suppose that the population consists in defiers and further suppose the treatment has a positive effect: $r_{\fD,0}=0$ and $r_{\fD,1}=1$.
	Bandit algorithms using protocol \#2 will learn to pull arm $c_1$, causing defiers to pull arm $0$. The best move in hindsight is the opposite.
\end{eg}
A population of defiers is arguably a pathological special case. The next scenario is more realistic in clinical trials:
\begin{eg}[Linear regret; harmful treatment]\label{eg:rich}\eod
	Suppose there are two sub-populations: the first consists of rich, healthy patients who always take the treatment. The second consists of poor, less healthy patients who only take the treatment if prescribed. Finally, suppose the treatment \emph{reduces} wellbeing by $0.25$ on some metric. We then have
	\begin{align}
	    \expec[R|a_0] & = p_\fC \cdot r_{\fC,0} = 0p_\fC \\
	    \expec[R|a_1] & = p_\fC \cdot r_{\fC,1} + p_\fA\cdot r_{\fA,1}
	    = -0.25p_\fC + 0.75p_\fA
	\end{align}
	If the population of healthy always-takers $p_\fA$ is sufficiently large, then \actual\, assigns higher rewards to the \emph{harmful} treatment arm.
\end{eg}

\paragraph{Protocol \#3: \comply.}
Finally, \chosen\, and \actual\, can be combined to form \comply, which only rewards an arm if (i) it was chosen by the bandit and (ii) the patient followed the bandit's advice.

\begin{prop}\label{prop:comply}
	There are settings where \comply\, outperforms \chosen\, and \actual.
\end{prop}

\begin{proof}
	It is easy to see that \comply\, outperforms \chosen\, in the setting of Proposition~\ref{prop:actual}.
	
	Consider a population of never-takers, always-takers and compliers. Suppose that never-takers are healthier than compliers $r_{\fN,0}> r_{\fC,0/1}$ whereas always-takers are less healthy $r_{\fA,1}<r_{\fC,0/1}$.  
    %The expected rewards received by \chosen\, are
	%\begin{align}
	%	\expec[R\,|\,c_0]
	%	& = p_{\fC}r_{\fC,0} + p_\fN r_{\fN,0} + p_\fA r_{\fA,1}
	%	\\
	%	\expec[R\,|\,c_1]
	%	& = p_{\fC}r_{\fC,1} + p_\fN r_{\fN,0} + p_\fA r_{\fA,1}
	%\end{align}
	
	Let $q_0$ and $q_1$ be the probability that the bandit pulls arms 0 and 1 respectively. The expected rewards received by \actual\, are
	\begin{align}
		\expec[R\,|\,a_0]
		& = \frac{q_0 p_{\fC}r_{\fC,0} + p_\fN r_{\fN,0}}{q_0p_\fC + p_\fN}
		\\
		\expec[R\,|\,a_1]
		& = \frac{p_\fA r_{\fA,1} + q_1p_\fC r_{\fC,1}}{p_\fA + q_1 p_\fC}
	\end{align}
	whereas the rewards used to update \comply\, are
	\begin{align}
		\expec[R\,|\,c_0,a_0] & = \frac{p_{\fC}r_{\fC,0} + p_\fN r_{\fN,0}}{p_\fC + p_\fN}
		\\
		\expec[R\,|\,c_1,a_1] & = \frac{p_\fA r_{\fA,1} + p_{\fC}r_{\fC,1}}{p_\fA + p_\fC}
	\end{align}
	It follows that 
	\begin{align}
		r_{\fC,0} &  \leq\expec[R\,|\,c_0,a_0]\leq \expec[R\,|\,a_0] \text{ and}\\
		r_{\fC,1} & \geq\expec[R\,|\,c_1,a_1] \geq \expec[R\,|\,a_1] 
	\end{align}
	The reward estimates for compliers are diluted under both \texttt{Actual} and \texttt{Comply}. However, \comply's estimate is more accurate.
\end{proof}

It is easy to see that \comply\, also has unbounded regret on example~\ref{eg:rich}.

The rewards assigned to each arm by the three protocols are summarized in the table below. None of the protocols successfully isolates the compliers. It follows, as seen above, that which protocol is optimal depends on the structure of the population, which is unknown to the learner.
\begin{center}
\begin{tabular}{| l | c | c | c |}
\hline
Arm updated & \chosen & \actual & \comply \\
\hline
      & $r_{\fN,0}$ & $r_{\fN,0}$ & $r_{\fN,0}$ \\
$i=0$ & $r_{\fC,0}$ & $r_{\fC,0}$ & $r_{\fC,0}$ \\
      & $r_{\fA,1}$ &             &             \\
      & $r_{\fD,1}$ & $r_{\fD,0}$ &             \\
\hline
      & $r_{\fN,0}$ &             &             \\
$i=1$ & $r_{\fC,1}$ & $r_{\fC,1}$ & $r_{\fC,1}$ \\
      & $r_{\fA,1}$ & $r_{\fA,1}$ & $r_{\fA,1}$ \\
      & $r_{\fD,0}$ & $r_{\fD,1}$ &             \\
\hline
\end{tabular}
\end{center}
The table can be extended with additional reward protocols. In this paper, we restrict attention to the three most intuitive protocols.

%%%%%%%%%%%%%%%%%%%%%%%%%%%%%%%%%%%%%%%%%%%%%%%%%%%%%%%%%%%%%%%%%%%%%%%%%%%%%%%%%%%%%%%%

%!TEX root = main.tex
%%%%%%%%%%%%%%%%%%%%%%%%%%%%%%%%%%%%%%%%%%%%%%%%%%%%%%%%%%%%%%%%%%%%%%%%%%%%%%%%%%%%%%%%
\section{Algorithms}
\label{sec:algol}

In the non-compliance setting there is additional information available to the algorithm. Ignoring the compliance-information (i.e. using the \chosen\, protocol) reduces to the standard bandit setting. However, it should be possible to improve performance by taking advantage of observations about when treatments are \emph{actually} applied. Using compliance-information is not trivial, since bandit algorithms that rely purely on treatments (\actual) or purely on compliance (\comply) can have linear regret.

This section proposes two hybrid algorithms that take advantage of compliance information, have bounded regret, and empirically outperform algorithms running the \chosen\, protocol.

%%%%%%%%%%%%%%%%%%%%%%%%%%%%%%%%%%%%%%%%%%%%%%%%%%%%%%%%%%%%%%%%%%%%%%%%%%%%%%%%%%%%%%%%
\subsection{Hierarchical bandits}

A natural idea is to use the three protocols to learn three experts and, simultaneously, learn which expert to apply. The result is a hierarchical bandit algorithm. The hierarchical bandit integrates compliance-information in a way that ensures the algorithm (i) has no-regret, because one of the base-algorithms uses \chosen, and therefore has no regret; and (ii) benefits from the compliance-information if it turns out to be useful.

The general construction is as follows. At the bottom-level are three bandit algorithms implementing the three protocols (\chosen, \actual\, and \comply). On the top-level is a fourth bandit algorithm whose arms are the three bottom-level algorithms. The bottom-level bandits optimally implement the three protocols, whereas the top-level bandit learns which protocol is optimal.

\begin{algorithm}[tb]
   \caption{\texttt{HierarchicalBandit (HB)}}
   \label{alg:hier-exp}
   \begin{algorithmic}   
   \STATE {\bfseries Input:}
   	 Bandits $\cB_i$ running \texttt{NoRegretAlgorithm} on \chosen, \actual\, and \comply\, for $i =\{1,2,3\}$ respectively, with arms corresponding to treatments
   	 \STATE {\bfseries Input:}
   	 Bandit $\cH$ running \texttt{NoRegretAlgorithm}, with arms corresponding to $\cB_i$ above
    \FOR{$t=1$ {\bfseries to} $T$}
	\STATE Draw bandit $i\in\{1,2,3\}$ from $\cH$ and arm $j$ from $\cB_i$
	\STATE Pull arm $j$; incur loss $\ell^{(t)}$; observe compliance
	\STATE Update $\cH$ with loss applied to bandit-arm $i$
	\IF{$i=1$}
	\STATE Update $\cB_1$ with loss applied to treatment-arm $j$
	\ENDIF
	\STATE Update $\cB_{2/3}$ with loss according to relevant protocol
   	\ENDFOR

       	\end{algorithmic}
\end{algorithm}          

The top-level bandit is \emph{not} in a stochastic environment even when the external environment is stochastic, since the low-level bandits are learning. We therefore use \texttt{EXP3} as the top-level bandit \cite{auer:02b}. 

\begin{thm}[No-regret with respect to \actual, \comply\, and individual treatment advice]\label{thm:cexp}\eod
	Let \texttt{EXP3} be the no-regret algorithm used in Algorithm~\ref{alg:hier-exp} for both the bottom and top-level bandits, with suitable choice of learning rate. Then, \texttt{HierarchicalBandit} satisfies
	\begin{equation}
		\expec\left[\sum_{t=1}^T\ell^{(t)}_{a^{(t)}}\right]
		\leq \sum_{t=1}^T \tilde{\ell}^{(t)}_{\actual/\comply}
		+ O(\sqrt{T})
	\end{equation}
	where $\tilde{\ell}^{(t)}_{\actual/\comply}$ denotes the expected loss vector of \texttt{EXP3} under the respective protocol on round $t$. 
	Furthermore, the regret against individual treatments $j\in[k]$ is bounded by
	\begin{equation}
		\expec\left[\sum_{t=1}^{T} \ell^{(t)}_{a^{(t)}}\right]
		\leq \sum_{t=1}^T\ell^{(t)}_{j}
		+ O(\sqrt{T k\log k})
	\end{equation}
\end{thm}

\begin{proof}
	Apply Lemma~\ref{lem:meta-exp} to $\texttt{HiearchicalBandit}$.
\end{proof}

Using  \texttt{EXP3} at the top-level and a \texttt{ThompsonSampler} in the  bottom-level also yields a no-regret algorithm. We modify the Thompson sampler to incorporate importance weighting, see Algorithm~\ref{alg:bts} in the Appendix.

\begin{algorithm}[tb]
   \caption{\texttt{ThompsonBounded (TB)}}
   \label{alg:dt}
   \begin{algorithmic}   
   \STATE {\bfseries Input:} Bandit algorithm $\cH$
   \STATE {\bfseries Input:} \texttt{Thompson} sampler under \chosen\, protocol
   	 \FOR{$t=1$ {\bfseries to} $T$}
	\STATE Sample $t$ and $t'$ from \texttt{Thompson}
	\IF{$t=t'$} 
	\STATE Pull arm sampled from \texttt{Thompson}
	\ELSE
	\STATE Pull arm chosen by $\cH$
	\ENDIF
	\STATE Incur loss, update algorithm used to pull arm
   	\ENDFOR
   	\end{algorithmic}
\end{algorithm}

%%%%%%%%%%%%%%%%%%%%%%%%%%%%%%%%%%%%%%%%%%%%%%%%%%%%%%%%%%%%%%%%%%%%%%%%%%%%%%%%%%%%%%%%
\subsection{Thompson bounding}

The second strategy starts from the observation that Thompson sampling often outperforms other bandit algorithms in stochastic settings \cite{thompson:33, chapelle:11} and has logarithmic regret \cite{agrawal:12, kaufmann:12}. A natural goal is then to design an algorithm that performs like Thompson sampling under the \chosen\, protocol in the long-run -- since Thompson sampling under \chosen\, is guaranteed to match the best action in hindsight in $O(\log T)$ time -- but also takes advantage of compliance side-information when Thompson sampling has \emph{not} converged onto sampling a single arm with high probability. 

The proposed algorithm, \texttt{ThompsonBounded}, adds an additional component to hierarchical bandit above: a Thompson sampler that learns from arm-pulls according to the \chosen\, protocol.
%distinguishing between chosen protocol and the exploration policy
The Thompson sampler is initially unbiased between arms; as it learns, the probabilities the Thompson sampler assigns to arms become increasingly concentrated. \texttt{ThompsonBounded} takes advantage of Thompson sampling's uncertainty about which arm to pull in early rounds to safely introduce side-information. To do so, \texttt{ThompsonBounded} draws two samples: if they agree, it plays a third Thompson sample. If they disagree it plays the arm chosen by the hierarchical bandit.

Intuitively, if Thompson sampling is uncertain, then \texttt{ThompsonBounded} tends to use the hierarchical bandit. As the Thompson sampler's confidence increases, \texttt{ThompsonBounded} is more likely to follow its advice. The next theorem shows that mixing in side information has no qualitative effect on the algorithm's regret, which grows as $\log(T)$. %The experiments in section~\ref{sec:results} show that \texttt{ThompsonBounded} performs significantly better than both the hierarchical bandit and the base protocols.

\begin{thm}\label{thm:tb}
	The regret of \texttt{ThompsonBounded} is bounded by
	\begin{equation}
		\regret_{\texttt{TB}}(T) \leq O(\log T).
	\end{equation}
\end{thm}

\begin{proof}	
	Suppose without loss of generality that arm 1 yields a higher average payoff. Let $p_j$ be the probability that \texttt{Thompson} assigns to arm $j$ on round $t$, so that $p_F=\sum_{j=2}^kp_j$ is the probability that Thompson sampling does \emph{not} pulling arm 1. The probability that \texttt{ThompsonBounded} follows the hierarchical bandit is then 
	$1-\sum_{j=1}^kp_j^2  = 
	    1-(1-p_F)^2 - \sum_{j=2}^{k-1}p_j^2 \leq 2p_F.$
	The additional expected regret from deviating from Thompson sampling is therefore at most twice the regret \texttt{Thompson} incurs by pulling suboptimal arms. Finally, it was shown in \cite{agrawal:12,kaufmann:12} that Thompson sampling has logarithmic regret.
\end{proof}

\subsection{Data-efficiency.}

The no-regret guarantees for \texttt{HB} and \texttt{TB} are provided, respectively, by the bottom-level expert running the \chosen\, protocol and the top-level Thompson sampler learning from the \chosen\, protocol. We refer to these strategies as \emph{certified}. The other strategies comprising the hybrids are not certified, but rather may boost empirical performance by bringing side-information from the compliance. 

As described, the hybrid algorithms are data-inefficient since, despite the i.i.d. assumption on the patient population, the certified strategies only learn when they are executed. We describe a \emph{recycling trick} to improve the efficiency of the certified strategies. 

A naive approach to increase data-efficiency is to reward the certified strategy on rounds where the executed strategy selects the same action as the certified strategy. However, this introduces a systematic bias. For example, consider two strategies: the first always picks arm 1, the second picks arms 1 and 2 with equal probability. Running a top-level algorithm that picks both with equal probability results in a mixed distribution biased towards arm 1.

The recycling trick stores actions and subsequent rewards by non-certified strategies in a cache. When there is at least one of each action in the cache, the certified strategy is rewarded on rounds where it was not executed by sampling, without replacement, from the cache. Sampling without replacement is important in our setting since it prevents early unrepresentative samples introducing a bias into the behavior of the certified strategy through repeated sampling. A related trick, referred to as ``experience replay'' was introduced in reinforcement learning in \cite{Mnih:2015wq}.

%, consider doing with replacement: a action that has been sampled one and received a low reward, can then make the arm it corresponds to arbitrarily bad, since it reduces the probability of further (real) samples of the arm in the future a unbounded numebr of times which it can be sampled.

%arm pull of the bounding algorithm or the base chosen algorithm can be rewarded form the cache. Note that this relies on a strong i.i.d. assumption about the underlying data generated process, but by making the cache go stale in time or when a shift in reward distribution is detected, other stochastic settings could be accommodated. 

\section{Clinical trial data}
\label{sec:data}

The simulation data is taken from The International Stroke Trial (IST) database. A randomised trial where patients believed to have acute ischaemic stroke are treated with: aspirin, subcutaneous heparin, both, or neither \cite{ist:97}.
Complete compliance and mortality data at 14 days for each of 19,422 patients
To the best of our knowledge, this is the largest publicly available clinical trial with compliance data.\footnote{An extensive search failed to find other open randomized clinical trials datasets that included compliance. A systematic review by  \cite{ebrahim:14} identified 37 reanalyses of patient-level data from previously published randomized control trials; five were performed by entirely independent authors.} Data from drug abuse clinical trials is used in \cite{kuleshov:14}. However, noncompliance is coded as failure so this source, and drug dependence treatments more generally, cannot be used in our setting. Given there is substantial loss of follow up at the 6 month measure we focus on the 14 day outcome.

\subsection{Compliance variables}

The main sources of noncompliance in the dataset are: the initial event not a stroke, clinical decision, administration problem, missed out more than 3 doses. A detailed table and counts of these are included in the datasets open access article \cite{ist:11}. 
While these might initially seem like reasons to discard the patients from the dataset, noncompliance is not necessarily random. Discarding these patients could cause algorithms to have unbounded regret (since the loss we care about is over all patients). In particular, misdiagnoses, administrative problems, not taking doses and other sources of noncompliance can be confounded with a patient's socioeconomic status, age, and overall health, as well as the load and quality of the medical staff. 

To construct our ``actual arm'' variable, we assume that noncompliance entails taking the opposite treatment.
This is well-defined in the Aspirin case, which only has two arms, and thus noncompliance with placebo is likely to be taking the treatment.

Assigning an actual arm pulled in the heparin part of the trial is less clear cut, as it has three arms: none, low and medium. We construct the actual arm variable by combining assignment and noncompliance. Noncompliance with respect to low and medium assigned treatments is coded as  not-takers, while noncompliance by a patient prescribed ``none'' is coded as low.

\section{Results}
\label{sec:results}

\paragraph{Stroke trial.}

\begin{figure}[ht!]
\vskip 0.2in
\begin{center}
\centerline{\includegraphics[width=0.5\textwidth]{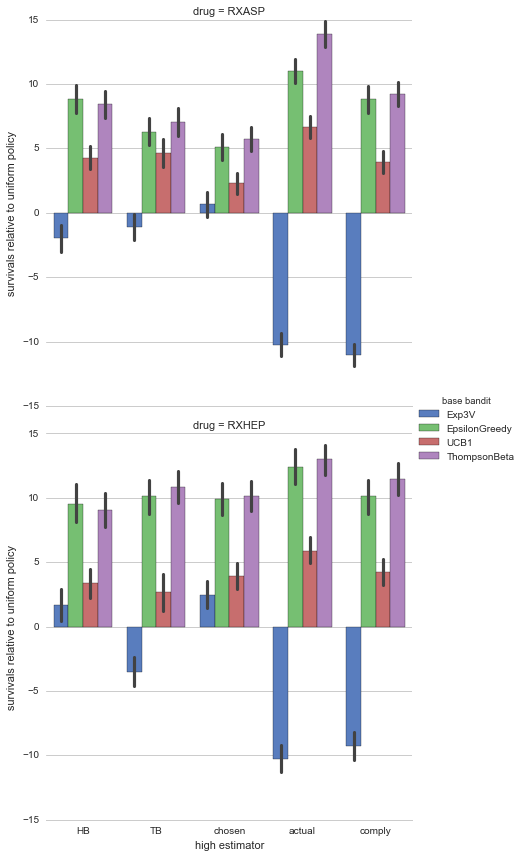}}
\caption{\textbf{14 Day survivals:} surplus over expectation of uniform random arm of a 10,000 patient each for simulated trials of Aspirin and Heparin.}
%\label{}
\end{center}
\vskip -0.2in
\end{figure}

In the stroke trial experiments, performance is measured in terms of the \emph{excess relative survival} or \emph{surplus} of different strategies. That is, the number of surviving patients in expectation, relative to a baseline that uniformly randomizes between treatment and control.

We simulate 10,000 patients per run, which allows us to not oversample the data in any single simulation; 2000 runs are performed for each algorithm.
The \texttt{EXP3} gamma parameter was set ahead of time to 0.085, a choice determined by the regret-bounds for $T=10000$ and $K=2$ or $3$. Epsilon-Greedy uses a standard annealing schedule. No data dependent parameter tuning was used.
The simulation is carried out by creating a ``counterfactual patient'' by sampling (i.i.d.) one patient from each of the treatment and control groups in the clinical trial. If the algorithm selects the treatment, it then receives the reward and observes the action taken by the subject sampled form the treatment group, and vice versa for the control.

Empirically, \texttt{ThompsonBounded} achieves a surplus of 8.9 extra survivals (that is, human lives) with 95\% confidence interval $[8.1,9.7]$, relative to the randomized baseline.
\texttt{HierarchicalBandit} with \texttt{Epsilon Greedy} as the base algorithm achieves a surplus of 9.2 (CI: $[8.3,10.0]$)
In contrast, the best performing strategy that is not compliance aware is Thompson sampling, which yields 7.9 extra survivals (CI: $[7.2,8.7]$). 

The gains are largely concentrated in the Aspirin trial, which is consistent with the lack of benefits or severe ill effects found in the original study \cite{ist:97} for heparin, and with the small but beneficial effect found for aspirin. 
If the underlying treatment has no positive or negative effect, side-information after the fact alone cannot be helpful.

Note that \actual, and to a lesser extent \comply, perform better than either \chosen\, or the hybrid algorithms. However, these cannot be used directly since no guarantees apply. The performance of the hybrids benefits from the information encoded in \actual\, and \comply\, whilst keeping the guarantees of \chosen.

%A striking secondary finding is the strong interaction between the learning algorithm and the nature of the feedback used. In particular, \texttt{EXP3} performs extremely badly under both \actual\, and \comply, in both the aspirin and heparin simulations. As a result, it is remarkable is that \texttt{EXP3} performs well when used as a top-level algorithm in \texttt{HB} in both trials, and other algorithms on the same data with the same procotol are able to learn better than chosen, while \texttt{EXP3} does worse than random (note the \texttt{EXP3} guarantees are only for the \chosen\, protocol, so do not apply here).

\paragraph{Synthetic data.}

To better understand the behaviour of the algorithms in a more varied range of settings, we present results of simulations with synthetic data.

\begin{figure}[ht!]
\label{fig:rich}
\vskip 0.2in
\begin{center}
\centerline{\includegraphics[width=0.5\textwidth]{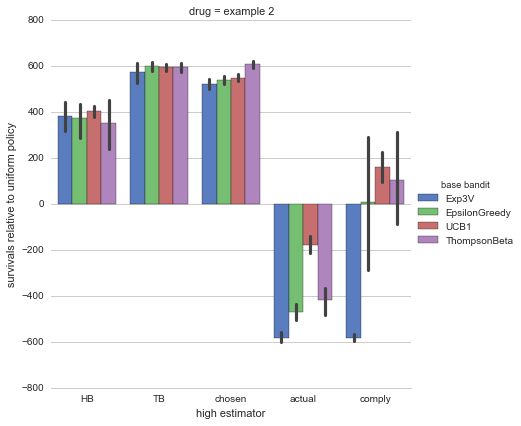}}
\caption{\textbf{Example 2 (rich and poor patients):} surplus over expectation of uniform random arm on 1,000 bootstrap samples simulating a 10,000  trial on each.}
%\label{}
\end{center}
\vskip -0.2in
\end{figure} 

The first simulation illustrates example~2. For comparison, $T$ is kept at 10,000, and consider the binary outcome case. We assign half the patients to rich and half to poor randomly. Rich patients always take the treatment and their outcome, which would otherwise be 1 with $p=1$, becomes 1 with  $p=0.75$. Poor patients only take the treatment when prescribed, their favorable outcome has probability $p=0.5$ without treatment; taking the treatment reduces the probability of a favorable outcome to $p=0.25$. Fig.~2 shows that the performance of \actual\, and \comply\, is much worse than \chosen\, and the hybrid algorithms.

\begin{figure}[ht!]
\label{fig:t12}
\vskip 0.2in
\begin{center}
\centerline{\includegraphics[width=0.5\textwidth]{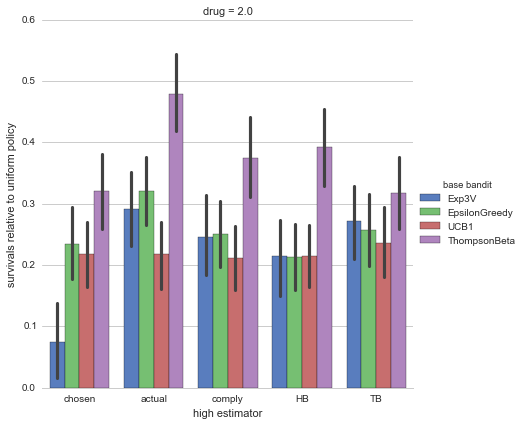}}
\caption{Expected surplus rewards relative to random assignment for an adaptive trial over 12 patients over 1,000 simulations.}
%\label{}
\end{center}
\vskip -0.2in
\end{figure}
The second simulation concerns small $T$. A motivation for very small $T$ adaptive clinical trials is provided by rare diseases. The overall size of the patient population is by construction severely restricted in this setting.  The priors for the mechanisms of action are also often poorly understood, so potential alternative treatments can have radically different probabilities of success. We simulate a $T=12$ adaptive trial with binary outcomes, with two treatments and expected rewards drawn uniformly from the unit interval, and compliance uniformly at random. We sample 1,000 such simulations. While our bounds are vacuous in this settings, it is interesting that there is on average an  improvement from taking the noncompliance information into account, see Fig.~3.

%think of the babies http://www.ncbi.nlm.nih.gov/pmc/articles/PMC2778326/

%We have two simulations, one in which noncompliance is independent of everything, and one where noncompliance is towards the higher expectation treatment.

%This is a boring figure, not rue what to say about it.

%%%%%%%%%%%%%%%%%%%%%%%%%%%%%%%%%%%%%%%%%%%%%%%%%%%%%%%%%%%%%%%%%%%%%%%%%%%%%%%%%%%%%%%%
\section{Conclusions}

%When learning, it can be useful to look not only at our recommendations, but also how those recommendations were carried out by others. This is the case even when our incentives and those of others involved are perfectly aligned. 
This paper introduced compliance information into the bandit setting. Compliance-information reflects the treatment actually taken by the patient, rather than the algorithm's recommendation. In many cases (perhaps most cases in practice) compliance information can be used to accelerate learning. Unfortunately, however, naively incorporating compliance information leads to algorithms with linear regret as seen in example~\ref{eg:rich} and figure~2. We have therefore developed hybrid strategies that are the first algorithms that simultaneously incorporate compliance information while maintaining a worst-case guarantee. 

Empirically, \texttt{TB} achieves a surplus of 8.9 extra survivals (that is, human lives) and \texttt{HB} achieves 9.2 surplus lives compared with 7.9 for the best classical algorithm. This suggests hybrid algorithms can make a significant difference to clinical outcomes.

%Compliance information will arise whenever bandit algorithms provide recommendations to humans, suggesting the setting will be of increasing importance in future. An important open problem is to characterize the \emph{advantage} that hybrid algorithms have over algorithms using the \chosen\, protocol, as a function of the structure of the compliance behavior. It is also likely that the algorithms proposed here are far from optimal, and that there is room to improve upon their performance.

%\bibliography{}
{
\footnotesize
\bibliography{compliance}

\begin{thebibliography}{26}
\providecommand{\natexlab}[1]{#1}
\providecommand{\url}[1]{\texttt{#1}}
\expandafter\ifx\csname urlstyle\endcsname\relax
  \providecommand{\doi}[1]{doi: #1}\else
  \providecommand{\doi}{doi: \begingroup \urlstyle{rm}\Url}\fi

\bibitem[Agrawal \& Goyal(2012)Agrawal and Goyal]{agrawal:12}
Agrawal, S and Goyal, N.
\newblock {Analysis of Thompson sampling for the multi-armed bandit problem}.
\newblock In \emph{Computational Learning Theory (COLT)}, 2012.

\bibitem[Alon et~al.(2015)Alon, Cesa-Bianchi, Dekel, and Koren]{alon:15}
Alon, Noga, Cesa-Bianchi, Nicol\'o, Dekel, Ofer, and Koren, Tomer.
\newblock {Online Learning with Feedback Graphs: Beyond Bandits}.
\newblock In \emph{Computational Learning Theory (COLT)}, 2015.

\bibitem[Auer(2002)]{auer:02}
Auer, Peter.
\newblock Using confidence bounds for exploitation-exploration trade-offs.
\newblock \emph{JMLR}, 3:\penalty0 397--422, 2002.

\bibitem[Auer et~al.(2002)Auer, Cesa-Bianchi, Freund, and Schapire]{auer:02b}
Auer, Peter, Cesa-Bianchi, Nicol\'o, Freund, Yoav, and Schapire, Robert.
\newblock The non-stochastic multi-armed bandit problem.
\newblock \emph{SIAM J. Computing}, 32\penalty0 (1):\penalty0 48--77, 2002.

\bibitem[Bareinboim et~al.(2015)Bareinboim, Forney, and Pearl]{bareinboim:15}
Bareinboim, Elias, Forney, Andrew, and Pearl, Judea.
\newblock {Bandits with Unobserved Confounders: A Causal Approach}.
\newblock In \emph{Adv in Neural Information Processing Systems (NIPS)}, 2015.

\bibitem[Bubeck(2012)]{bubeck:12}
Bubeck, S{\'e}bastien.
\newblock Regret analysis of stochastic and nonstochastic multi-armed bandit
  problems.
\newblock \emph{Foundations and Trends in Machine Learning}, 5\penalty0
  (1):\penalty0 1--122, 2012.

\bibitem[Bubeck \& Slivkins(2012)Bubeck and Slivkins]{bubeck:12a}
Bubeck, S\'ebastien and Slivkins, Aleksandrs.
\newblock The best of both worlds: stochastic and adversarial bandits.
\newblock In \emph{Computational Learning Theory (COLT)}, 2012.

\bibitem[Chang \& Kaelbling(2005)Chang and Kaelbling]{chang:05}
Chang, Yu-Han and Kaelbling, Leslie~Pack.
\newblock {Hedged learning: Regret-minimization with learning experts}.
\newblock In \emph{ICML}, 2005.

\bibitem[Chapelle \& Li(2011)Chapelle and Li]{chapelle:11}
Chapelle, Olivier and Li, Lihong.
\newblock {An Empirical Evaluation of Thompson Sampling}.
\newblock In \emph{Adv in Neural Information Processing Systems (NIPS)}, 2011.

\bibitem[Ebrahim et~al.(2014)Ebrahim, Sohani, Montoya, Agarwal, Thorlund,
  Mills, and Ioannidis]{ebrahim:14}
Ebrahim, Shanil, Sohani, Zahra~N, Montoya, Luis, Agarwal, Arnav, Thorlund,
  Kristian, Mills, Edward~J, and Ioannidis, John~PA.
\newblock Reanalyses of randomized clinical trial data.
\newblock \emph{Jama}, 312\penalty0 (10):\penalty0 1024--1032, 2014.

\bibitem[Graepel et~al.(2010)Graepel, Quionero-Candela, Borchert, and
  Herbrich]{graepel:10}
Graepel, T, Quionero-Candela, J, Borchert, T, and Herbrich, R.
\newblock {Web-scale Bayesian click-through rate prediction for sponsored
  search and advertising in Microsoft's Bing engine}.
\newblock In \emph{ICML}, 2010.

\bibitem[Group(1997)]{ist:97}
Group, International Stroke Trial~Collaborative.
\newblock The international stroke trial (ist): a randomised trial of aspirin,
  subcutaneous heparin, both, or neither among 19 435 patients with acute
  ischaemic stroke.
\newblock \emph{The Lancet}, 349\penalty0 (9065):\penalty0 1569--1581, 1997.

\bibitem[Hugtenburg et~al.(2013)Hugtenburg, Timmers, Elders, Vervloet, and van
  Dijk]{hugtenburg:13}
Hugtenburg, Jacqueline~G, Timmers, Lonneke, Elders, PJ, Vervloet, Marcia, and
  van Dijk, Liset.
\newblock Definitions, variants, and causes of nonadherence with medication: a
  challenge for tailored interventions.
\newblock \emph{Patient Prefer Adherence}, 7:\penalty0 675--682, 2013.

\bibitem[Kaufmann et~al.(2012)Kaufmann, Korda, and Munos]{kaufmann:12}
Kaufmann, E, Korda, N, and Munos, R.
\newblock {Thompson sampling: An asymptotically optimal finite-time analysis}.
\newblock In \emph{ALT}, 2012.

\bibitem[Koller \& Friedman(2009)Koller and Friedman]{koller:09}
Koller, Daphne and Friedman, Nir.
\newblock \emph{{Probabilistic Graphical Models: Principles and Techniques}}.
\newblock MIT Press, 2009.

\bibitem[Kuleshov \& Precup(2014)Kuleshov and Precup]{kuleshov:14}
Kuleshov, Volodymyr and Precup, Doina.
\newblock Algorithms for multi-armed bandit problems.
\newblock \emph{arXiv preprint arXiv:1402.6028}, 2014.

\bibitem[Lai \& Robbins(1985)Lai and Robbins]{lai:85}
Lai, T~L and Robbins, H.
\newblock Asymptotically efficient adaptive allocation rules.
\newblock \emph{Advances in Applied Mathematics}, 6:\penalty0 4--22, 1985.

\bibitem[McMahan et~al.(2013)McMahan, Holt, Sculley, Young, Ebner, Grady, Nie,
  Phillips, Davydov, Golovin, Chikkerur, Liu, Wattenberg, Hrafnkelsson, Boulos,
  and Kubica]{mcmahan:13}
McMahan, H.~Brendan, Holt, Gary, Sculley, D., Young, Michael, Ebner, Dietmar,
  Grady, Julian, Nie, Lan, Phillips, Todd, Davydov, Eugene, Golovin, Daniel,
  Chikkerur, Sharat, Liu, Dan, Wattenberg, Martin, Hrafnkelsson, Arnar~Mar,
  Boulos, Tom, and Kubica, Jeremy.
\newblock Ad click prediction: A view from the trenches.
\newblock In \emph{Proceedings of the 19th ACM SIGKDD International Conference
  on Knowledge Discovery and Data Mining}, New York, NY, USA, 2013.

\bibitem[Mnih et~al.(2015)Mnih, Kavukcuoglu, Silver, Rusu, Veness, Bellemare,
  Graves, Riedmiller, Fidjeland, Ostrovski, Petersen, Beattie, Sadik,
  Antonoglou, King, Kumaran, Wierstra, Legg, and Hassabis]{Mnih:2015wq}
Mnih, Volodymyr, Kavukcuoglu, Koray, Silver, David, Rusu, Andrei~A., Veness,
  Joel, Bellemare, Marc~G., Graves, Alex, Riedmiller, Martin, Fidjeland,
  Andreas~K., Ostrovski, Georg, Petersen, Stig, Beattie, Charles, Sadik, Amir,
  Antonoglou, Ioannis, King, Helen, Kumaran, Dharshan, Wierstra, Daan, Legg,
  Shane, and Hassabis, Demis.
\newblock Human-level control through deep reinforcement learning.
\newblock \emph{Nature}, 518\penalty0 (7540):\penalty0 529--533, 02 2015.

\bibitem[Robbins(1952)]{robbins:52}
Robbins, H.
\newblock Some aspects of the sequential design of experiments.
\newblock \emph{Bull. AMS}, 58:\penalty0 527--535, 1952.

\bibitem[Sabat{\'e}(2003)]{sabate:03}
Sabat{\'e}, Eduardo.
\newblock \emph{Adherence to long-term therapies: evidence for action}.
\newblock World Health Organization, 2003.

\bibitem[Sandercock et~al.(2011)Sandercock, Niewada, Cz{\l}onkowska,
  et~al.]{ist:11}
Sandercock, Peter~AG, Niewada, Maciej, Cz{\l}onkowska, Anna, et~al.
\newblock The international stroke trial database.
\newblock \emph{Trials}, 12\penalty0 (1):\penalty0 1--7, 2011.

\bibitem[Seldin \& Slivkins(2014)Seldin and Slivkins]{seldin:14}
Seldin, Yevgeny and Slivkins, Aleksandrs.
\newblock {One Practical Algorithm for Both Stochastic and Adversarial
  Bandits}.
\newblock In \emph{ICML}, 2014.

\bibitem[Thompson(1933)]{thompson:33}
Thompson, William~R.
\newblock On the likelihood that one unknown probability exceeds another in
  view of the evidence of two samples.
\newblock \emph{Biometrika}, 25\penalty0 (3/4):\penalty0 285--294, 1933.

\bibitem[Vapnik \& Vashist(2009)Vapnik and Vashist]{vapnik:09}
Vapnik, Vladimir and Vashist, Akshay.
\newblock {A new learning paradigm: Learning using privileged information}.
\newblock \emph{Neural Netw}, 22:\penalty0 544--557, 2009.

\bibitem[Vrijens et~al.(2012)Vrijens, De~Geest, Hughes, Przemyslaw, Demonceau,
  Ruppar, Dobbels, Fargher, Morrison, Lewek, et~al.]{vrijens:12}
Vrijens, Bernard, De~Geest, Sabina, Hughes, Dyfrig~A, Przemyslaw, Kardas,
  Demonceau, Jenny, Ruppar, Todd, Dobbels, Fabienne, Fargher, Emily, Morrison,
  Valerie, Lewek, Pawel, et~al.
\newblock A new taxonomy for describing and defining adherence to medications.
\newblock \emph{British journal of clinical pharmacology}, 73\penalty0
  (5):\penalty0 691--705, 2012.

\end{thebibliography}
\bibliographystyle{icml2016}
}

\clearpage
%!TEX root = main.tex
%%%%%%%%%%%%%%%%%%%%%%%%%%%%%%%%%%%%%%%%%%%%%%%%%%%%%%%%%%%%%%%%%%%%%%%%%%%%%%%%%%%%%%%%
\section{No-regret for \texttt{HierarchicalBandit}}
\label{sec:bound}

This section shows that constructing a hierarchical bandit with \texttt{EXP3} yields a no-regret algorithm. The result is straightforward; we include it for completeness. A similar result was shown in \cite{chang:05}. 

First, we construct a hierarchical version of \texttt{Hedge}, Algorithm~\ref{alg:meta-hedge}, which is applicable in the full-information setting. On the bottom-level are $M$ instantiations of \texttt{Hedge}. Instantiation $i$, for $i\in[M]$, plays an $N$-dimensional weight vector and receives $N$-dimensional loss vector $\loss^{(t)}_{i}$ on round $t$. We impose the assumption that all instantiations play $N$-vectors for notational convenience. The top-level is another instantiation of \texttt{Hedge}, which plays a weighted combination of the bottom-level instantiations.

\begin{algorithm}[tb]
   \caption{\texttt{Hierarchical Hedge (HHedge)}}
   \label{alg:meta-hedge}
   \begin{algorithmic}   
   	\STATE {\bfseries Input:} $\eta,\rho>0$\\
   	 $v^{(1)}_{i}=1$ for $i\in[M]$;\\ 
   	 $w^{(1)}_{i,j}=1$ for $(i,j)\in[M]\times[N]$
	\FOR{$t=1$ {\bfseries to} $T$}
	\STATE Set $\x^{(t)} \leftarrow \vt^{(t)}/X^{(t)}$ where $X^{(t)} = \sum_{i=1}^M v^{(t)}_{i}$.
	\STATE Set $\y^{(t)}_{i} \leftarrow \wt^{(t)}_{i}/Y^{(t)}_{i}$ where $Y^{(t)}_{i} = \sum_{j=1}^N w^{(t)}_{i,j}$.
	\STATE Receive feedback $\loss^{(t)}\in [0,1]^{M\times N}$ 
	\STATE Incur loss $\sum_{i,j=1}^{M,N} x^{(t)}_{i}\cdot\ell^{(t)}_{i,j}\cdot y^{(t)}_{i,j}$
	\STATE Updates:
	\begin{align}
		v^{(t+1)}_i & \leftarrow v^{(t)}_{i}\cdot \exp\big(-\eta \sum_{j=1}^N\ell^{(t)}_{i,j}\cdot y^{(t)}_{i,j}\big)
		\\
		w^{(t+1)}_{i,j} & \leftarrow w^{(t)}_{i,j}\cdot \exp\big(-\rho\cdot \ell^{(t)}_{i,j}\big)
	\end{align}
   	\ENDFOR
   	\end{algorithmic}
\end{algorithm}

We have the following lemma:

\begin{lem}\label{lem:meta-hedge}
	Introduce compound loss vector $\tilde{\loss}$ with $\tilde{\ell}^{(t)}_i := \sum_j \ell^{(t)}_{i,j}\cdot y^{(t)}_{i,j}$. Then $\rho$ can be chosen in \texttt{HHedge} such that
	\begin{equation}
		\sum_{t=1}^T \langle \x^{(t)},\tilde{\loss}^{(t)}\rangle 
		\leq  \sum_{t=1}^T \tilde{\loss}^{(t)}_i +
		\leq O(\sqrt{T \log M})\quad\forall i.
	\end{equation}
	Moreover, $\rho$ and $\eta$ can be chosen such that, for all $i$,
	\begin{equation}
		\sum_{t,i,j=1}^{T,M,N} x^{(t)}_{i}\ell^{(t)}_{i,j} y^{(t)}_{i,j}
		\leq \sum_{t=1}^T\ell^{(t)}_{i,j}
		+ O(\sqrt{T \log M} + \sqrt{T \log N}).
	\end{equation}
\end{lem}

\begin{proof}
	Apply regret bounds for \texttt{Hedge} twice.
\end{proof}

Lemma~\ref{lem:meta-hedge} says, firstly, that \texttt{HHedge} has bounded regret relative to the bottom-level instantiations and, secondly, that it has bounded regret relative to any of the $M\times N$ experts on the bottom-level.

Algorithm~\ref{alg:meta-exp2} modifies \texttt{HHedge} so that it is suitable for bandit feedback, yielding \texttt{HEXP3}. A corresponding no-regret bound follows immediately:

\begin{lem}\label{lem:meta-exp}
	Define $\tilde{\loss}$ as in Lemma~\ref{lem:meta-hedge}. Then $\rho$ can be chosen in \texttt{HEXP3} such that
	\begin{equation}
		\expec\left[\sum_{t=1}^T\ell_{x^{(t)},y^{(t)}}\right]
		\leq \sum_{t=1}^T \tilde{\ell}^{(t)}_{i}
		+ O(\sqrt{MT\log M})
	\end{equation}
	Moreover, $\rho$ and $\eta$ can be chosen such that
	\begin{equation}
		\expec\left[\sum_{t=1}^{T} \ell^{(t)}_{x^{(t)},y^{(t)}}\right]
		\leq \sum_{t=1}^T\ell^{(t)}_{i,j}
		+ O(\sqrt{TM \log M} + \sqrt{T N\log N})
	\end{equation}
\end{lem}
\begin{proof}
	Follows from Lemma~\ref{lem:meta-hedge} and bounds for \texttt{EXP3}.
\end{proof}

\begin{algorithm}[tb]
   \caption{\texttt{Hierarchical EXP3 (HEXP3)}}
   \label{alg:meta-exp2}
   \begin{algorithmic}   
   \STATE {\bfseries Input:} $\eta,\rho>0$\\
   	 $v^{(1)}_{i}=1$ for $i\in[M]$;\\ 
   	 $w^{(1)}_{i,j}=1$ for $(i,j)\in[M]\times[N]$
	\FOR{$t=1$ {\bfseries to} $T$}
	\STATE Set $\x^{(t)} \leftarrow \vt^{(t)}/X^{(t)}$ where $X^{(t)} = \sum_{i=1}^M v^{(t)}_{i}$.
	\STATE Set $\y^{(t)}_{i} \leftarrow \wt^{(t)}_{i}/Y^{(t)}_{i}$ where $Y^{(t)}_{i} = \sum_{j=1}^N w^{(t)}_{i,j}$.
	\STATE Draw $x^{(t)}\sim \x^{(t)}$ and $y^{(t)}\sim \y^{(t)}_{x^{(t)}}$.
	\STATE Incur loss $\ell^{(t)}_{x^{(t)},y^{(t)}}\in [0,1]$ 
	\STATE Updates:
	\begin{align}
		v^{(t+1)}_i & \leftarrow \begin{cases}
			v^{(t)}_{i}\cdot 
			\exp\big(-\eta\frac{\ell^{(t)}_{i,j}}{x_i}\big) & i=x^{(t)} \\
			v^{(t)}_{i} & \text{else}
		\end{cases}		 
		\\
		w^{(t+1)}_{i,j} & \leftarrow \begin{cases}
			w^{(t)}_{i,j}\cdot \exp\big(-\rho\frac{\ell^{(t)}_{i,j}}{x_iy_{i,j}}\big) 
			& \text{if }(i,j)=(x^{(t)}, y^{(t)}) \\
			w^{(t)}_{i,j} &\text{else}
		\end{cases}
	\end{align}
   	\ENDFOR
   	\end{algorithmic}
\end{algorithm}

\paragraph{Hierarchical Bandit with Thompson sampler base.}
Algorithm~\ref{alg:bts} (\texttt{BTS}) shows how to modify the Thompson sampler for use as a bottom-level algorithm in \texttt{HierarchicalBandit}. The modification applies the importance weighting trick: replace $1$ in Thompson sampling with $\tilde{1}=1/p$, where $p$ is the probability that the top-level bandit calls \texttt{BTS} on the given round.

\begin{algorithm}[tb]
   \caption{\texttt{Base Thompson Sampler (BTS)}}
   \label{alg:bts}
   \begin{algorithmic}
   	\STATE {\bfseries Input:} Probability $p$ that \texttt{BTS} is called by top-bandit\\
   	\STATE Set $\tilde{1}\leftarrow 1/p$
   	 \STATE For each arm $i$ sample $\theta_i\sim\beta(S_i +\tilde{1},F_i +\tilde{1})$
	\STATE Play arm $i^{(t)} := \argmax_i \theta_i$ and observe reward $r^{(t)}$
	\STATE Sample $b$ from Bernoulli with success probability $r^{(t)}$
	\STATE If $b=1$ then $S_i \leftarrow S_i + \tilde{1}$ else $F_i\leftarrow F_i+\tilde{1}$
   	\end{algorithmic}
\end{algorithm}

\end{document}